\documentclass{article} 
\usepackage{iclr2027_conference,times}

\usepackage{hyperref}
\usepackage{url}

\usepackage{booktabs}       
\usepackage{amsfonts}       
\usepackage{nicefrac}       
\usepackage{microtype}      
\usepackage{xcolor}         

\def\b{{\boldsymbol b}}

\def\h{{\boldsymbol h}}
\def\I{{\boldsymbol I}}

\def\P{{\boldsymbol P}}
\def\p{{\boldsymbol p}}

\def\U{{\boldsymbol U}}

\def\W{{\boldsymbol W}}

\def\x{{\boldsymbol x}}

\def\z{{\boldsymbol z}}

\def\DM{{\mathcal D}}

\def\FM{{\mathcal F}}

\def\LM{{\mathcal L}}

\def\NM{{\mathcal N}}

\def\SM{{\mathcal S}}
\def\TM{{\mathcal T}}

\def\WM{{\mathcal W}}

\def\ZM{{\mathcal Z}}

\def\RB{{\mathbb R}}

\usepackage{graphicx}
\usepackage{wrapfig}
\usepackage{amsmath}
\usepackage{amssymb}
\usepackage{algorithm,algorithmic}
\usepackage{rotating}
\usepackage{multirow} 
\usepackage{pifont}
\usepackage{dsfont}
\usepackage{colortbl}
\usepackage{tikz}
\usepackage{subcaption}
\usepackage{amsthm}
\usepackage{titletoc}

\newtheorem{theorem}{Theorem}

\title{Representation Editing for Multimodal Test-Time Adaptation}

\author{\textbf{Longfei Huang}$^{1}$\thanks{Equal contribution.}~~\ \
\textbf{Xiangyu Wu}$^{2*}$\ \
\textbf{Yang Yang}$^{1}$\thanks{Corresponding author.} \\
$^{1}$Nanjing University of Science and Technology\\
$^{2}$Alibaba Group
}

\iclrfinalcopy 
\begin{document}

\maketitle

\begin{abstract}
Multimodal test-time adaptation (TTA) aims to adapt a pretrained multimodal model online to distribution shift across modalities using unlabeled test data, showing broad potential in real-world applications. However, existing methods primarily focus on adjusting fused features to bridge the source-target gap, lacking explicit control over intermediate representation misalignment, which is a key driver of performance drop under distribution shift. In this work, we tackle this challenge from the perspective of representation engineering. Unlike previous TTA methods that update fusion weights in place, we propose FourIer Representation Editor (FIRE), a novel multimodal TTA approach that directly edits semantically rich intermediate representations. Specifically, we first adopt representation editors into each intermediate layer of the unimodal encoders, enabling layer-wise calibration of unimodal representations. To further enhance the diversity and stability of the low-rank editing subspaces, each representation editor performs frequency domain mixing via the fast Fourier transform to construct structured bases. Moreover, we introduce multi-level adaptation objectives to optimize these editors, jointly promoting cross-modal semantic alignment, source-target statistical alignment, and asymmetric prediction consistency. In this way, FIRE yields aligned unimodal representations for fusion and further improves prediction reliability. Extensive experiments on two widely used multimodal benchmarks under various corruption types demonstrate the superiority of FIRE over existing multimodal TTA methods.
\end{abstract}

\section{Introduction}
Multimodal learning \cite{Hybrid:conf/aaai/ZhengTWH023,MML:conf/iccv/Rakib_2025_ICCV,MML:conf/nips/JiangHY25,MML:conf/aaai/00030000N25,MML:journals/pami/DongLZCKSF26,MML:conf/aaai/HuJCWSSX26} aims to integrate heterogeneous information from different sensors to improve holistic task understanding, which has received growing attention in various real-world applications. Despite its success, multimodal models often suffer from distribution shifts between source and target data in open-world scenarios due to unpredictable factors, such as weather variations and sensor failures. Consequently, the performance of static models can degrade substantially during inference \cite{DBLP:conf/iclr/HendrycksD19,TTA:conf/cvpr/ShinTZSLGKY22,TTA:journals/corr/abs-2604-19093}. To address this issue, test-time adaptation (TTA) \cite{Tent:conf/iclr/WangSLOD21,TTA:conf/aaai/ZhangCLDYLXDWD26} has emerged as a practical and effective solution: it updates model parameters online using unlabeled test data to adapt to the target distribution, transforming static inference into an adaptive process.

Most existing TTA methods, such as Tent \cite{Tent:conf/iclr/WangSLOD21}, EATA \cite{EATA:conf/icml/NiuW0CZZT22}, and SAR \cite{SAR:conf/iclr/Niu00WCZT23}, focus on unimodal settings and fine-tuning normalization layers via entropy minimization. However, in practice, different modalities may experience varying degrees of distribution shift in multimodal learning, making multimodal TTA substantially more challenging. Specifically, multimodal systems may encounter either corruption in a single modality or simultaneous corruption across multiple modalities. Fortunately, several studies \cite{READ:conf/iclr/Yang0Z0024,TTA:journals/corr/abs-2603-00574} have investigated multimodal distribution shifts. As a pioneering work on multimodal TTA, READ \cite{READ:conf/iclr/Yang0Z0024} and its successors ABPEM \cite{ABPEM:conf/aaai/Zhao0LHY0025} and TSA \cite{TSA:conf/icml/ChenZHJWF0B25}, etc., address modality-specific corruption by updating the self-attention layer in the fusion module. Recent research, including SuMi \cite{SuMi:conf/iclr/Guo025} and PTA \cite{PTA:wang2026partition}, further considers more challenging scenarios where multiple modalities are corrupted simultaneously.

\begin{figure}[t]
    \centering
    \captionsetup[subfigure]{
        font=small,
        justification=centering,
        singlelinecheck=true
    }
    \begin{subfigure}[t]{0.24\textwidth}
        \centering
        \begin{minipage}[c][3.0cm][c]{\linewidth}
            \centering
            \includegraphics[width=0.82\linewidth]{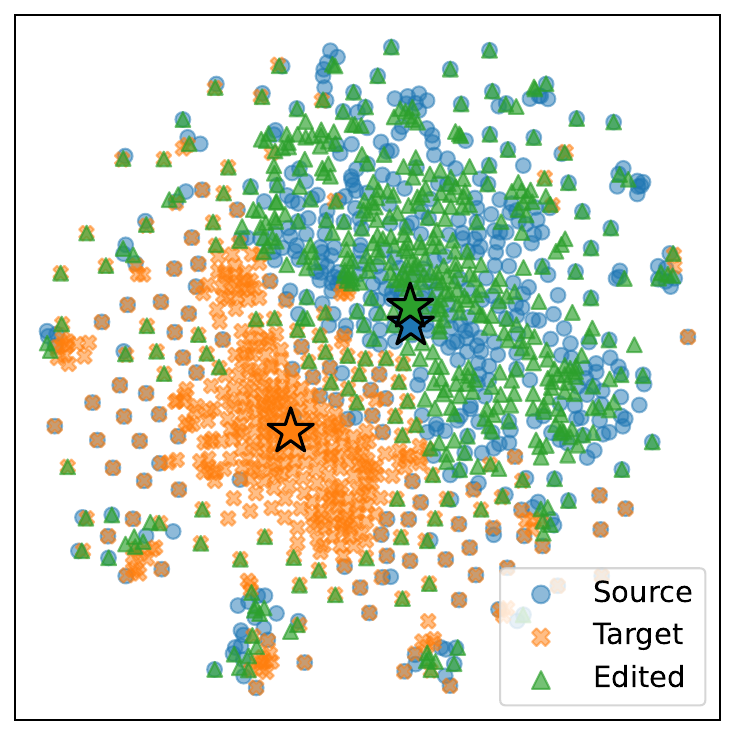}
        \end{minipage}
        \caption{Representation shift.}
        \label{fig:intro_shift}
    \end{subfigure}
    \hfill
    \begin{subfigure}[t]{0.24\textwidth}
        \centering
        \begin{minipage}[c][3.0cm][c]{\linewidth}
            \centering
            \includegraphics[width=0.98\linewidth]{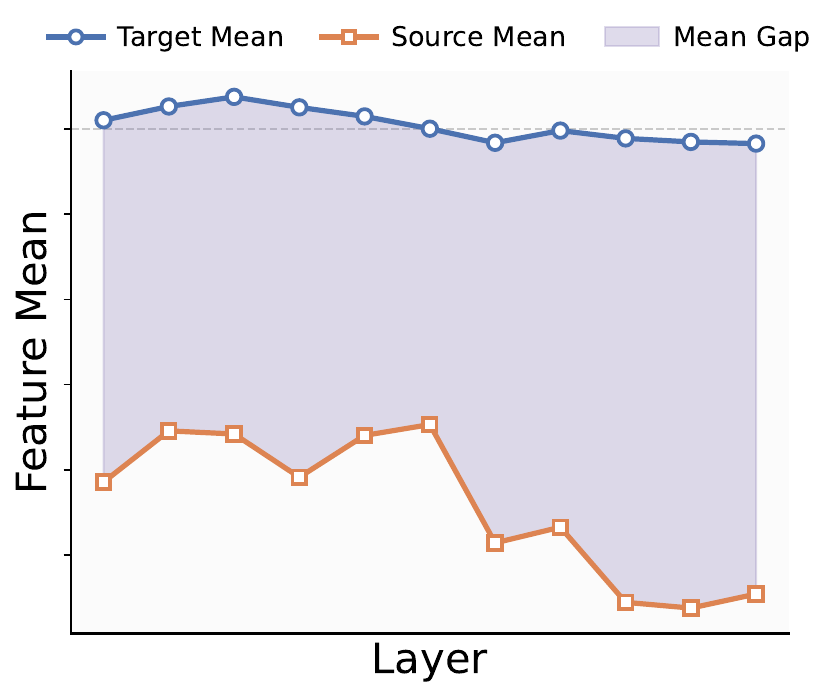}
        \end{minipage}
        \caption{Mean misalignment.}
        \label{fig:intro_mean}
    \end{subfigure}
    \hfill
    \begin{subfigure}[t]{0.24\textwidth}
        \centering
        \begin{minipage}[c][3.0cm][c]{\linewidth}
            \centering
            \includegraphics[width=0.98\linewidth]{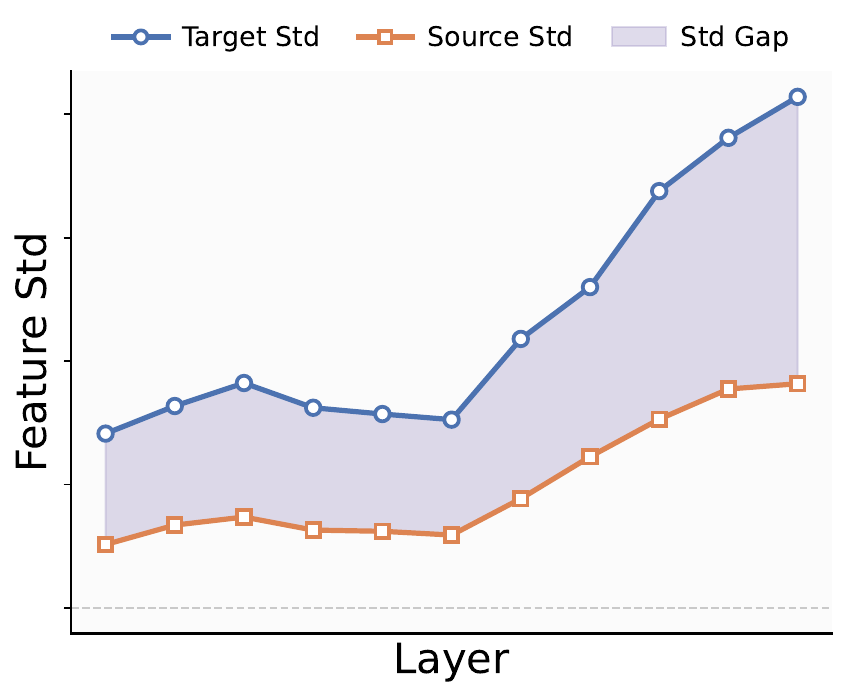}
        \end{minipage}
        \caption{Std. misalignment.}
        \label{fig:intro_std}
    \end{subfigure}
    \hfill
    \begin{subfigure}[t]{0.26\textwidth}
        \centering
        \begin{minipage}[c][3.0cm][c]{\linewidth}
            \centering
            \includegraphics[width=\linewidth]{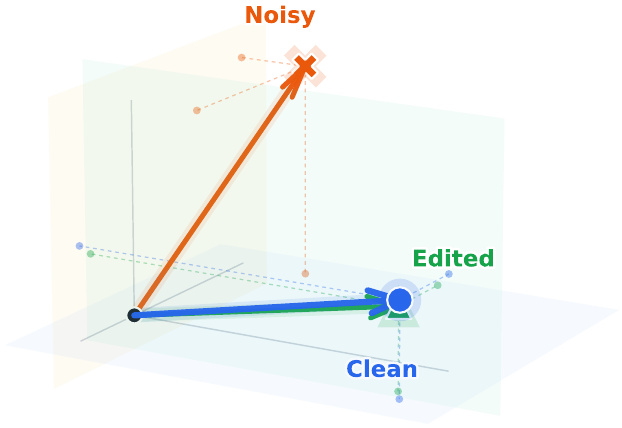}
        \end{minipage}
        \caption{Representation editing.}
        \label{fig:intro_edit}
    \end{subfigure}
    \caption{
        Overview of representation shift and editing on Kinetics50-C.
        (a) t-SNE visualization of multimodal representations from the source,
        target, and edited target domains. The star denotes the feature centroid.
        (b) and (c) Differences in intermediate-layer feature statistics between
        the source and target domains, where the mean characterizes the
        distribution center and the standard deviation (std) characterizes its
        dispersion.
        (d) Visualization of clean, noisy, and edited representations of the
        same target sample. Our method effectively mitigates the representation shift.
    }
    \label{fig:intro}
\end{figure}

Although the aforementioned methods can alleviate multimodal distribution shift, they focus more on adapting fused multimodal features while lacking explicit control over the representation shifts in intermediate layers, which serve as a precursor to multimodal distribution shift. Unlike previous works, we seek a solution from the perspective of representation engineering. As shown in Figure \ref{fig:intro} (a), we visualize multimodal features from the source and target domains on Kinetics50-C \cite{READ:conf/iclr/Yang0Z0024} dataset using t-SNE \cite{tSNE:journal/jmlr/MaatenH08} and observe a clear separation between their distributions, suggesting that representation shifts are the most direct manifestation of distribution shifts. Since representation shifts typically do not emerge abruptly at the fusion stage, we believe that the misalignment in intermediate representations is a key driver of performance drop under distribution shift. Therefore, we quantify the statistical discrepancy between source and target features in intermediate layers. As shown in Figure \ref{fig:intro} (b) and (c), the results indicate that the shift emerges in shallow layers and tends to increase with depth. This naturally raises the question: can we directly modify intermediate representations to align the target with the source domain, thereby enhancing the quality of fusion-layer representation and improving prediction reliability?

Motivated by the above insights, we propose a novel Multimodal TTA approach, FourIer Representation Editor (FIRE), to directly edit semantically rich intermediate representations, facilitating alignment between the source and target domains. Specifically, we first introduce representation editors and place them after each transformer layer in every unimodal encoder, jointly optimizing unimodal representations layer-wise. By layer-wise calibrating unimodal representations, editors align the target distribution with the source, thereby facilitating reliable fusion. Notably, only the editors are updated in TTA, while the pretrained model remains entirely frozen. Subsequently, we enhance the representation editor with fast Fourier transforms \cite{FFT:reddy1998fast,FFT:conf/nips/ChiJM20}. By incorporating frequency domain information into the low-rank subspace, FIRE integrates global coupling across dimensions, enriching the editing directions and improving adaptation stability. Furthermore, we optimize these editors with multi-level adaptation objectives, including cross-channel semantic alignment, source-target statistical calibration, and asymmetric prediction consistency. To this end, FIRE effectively calibrates intermediate representations and provides aligned representations for fusion, thereby improving prediction reliability. As shown in Figure \ref{fig:intro}(a) and (d), FIRE significantly realigns the representation distribution. Extensive experiments on widely used multimodal TTA benchmarks demonstrate that FIRE achieves significant performance improvements.

\section{Related work}
\subsection{Multimodal Test-Time Adaptation}
Test-time adaptation (TTA) \cite{TTA:conf/iclr/Li0Z0XY25,TTA:conf/iclr/WuY0CL25,TTA:conf/nips/JiaDTLZL25,TTA:conf/aaai/FanJCHCJZTW26,TTA:journals/corr/abs-2602-11743} aims to online adapt a model to the current target distribution using unlabeled test samples. Most TTA methods mainly focus on unimodal learning settings \cite{TTA:journals/pami/LiYDZS24,TTA:conf/iclr/ZhangBKZZ25,TTA:conf/aaai/ChuYWW26,TTA:conf/aaai/YangLW26}. A representative method, TENT \cite{Tent:conf/iclr/WangSLOD21}, updates the normalization layers by minimizing prediction entropy. Subsequent works, including EATA \cite{EATA:conf/icml/NiuW0CZZT22}, SAR \cite{SAR:conf/iclr/Niu00WCZT23}, and DeYO \cite{DeYO:conf/iclr/0004JLPSHY24}, design sample filtering strategies to promote adaptation with more informative target samples.

Multimodal TTA is more challenging, as different modalities may suffer from corruptions of varying severity. As a pioneering work on multimodal TTA, READ \cite{READ:conf/iclr/Yang0Z0024} highlights the reliability bias caused by unimodal corruption and addresses it by optimizing a self-attention \cite{Attn:conf/nips/VaswaniSPUJGKP17} fusion layer. Its successors further improve multimodal TTA with strategies such as attention bootstrapping \cite{ABPEM:conf/aaai/Zhao0LHY0025}, smoothing the shift \cite{SuMi:conf/iclr/Guo025}, and partition-then-adapt \cite{PTA:wang2026partition}. Other works \cite{TSA:conf/icml/ChenZHJWF0B25,BriMPR:conf/aaai/LiF26} begin to focus on unimodal information, aiming to bridge the source-target gap by correcting unimodal shifts before fusion. However, these methods lack explicit control over intermediate representation shifts. In this paper, we aim to address distribution shifts by editing unimodal representations in a layer-wise manner.

\subsection{Representation Engineering}
Representation engineering \cite{ReE:journals/corr/abs-2310-01405,ReE:conf/acl/SubramaniSP22,ReE:conf/iclr/BurnsYKS23,ReE:journals/corr/abs-2502-17601} aims to understand and control model behavior from the perspective of internal neural representations. Recent studies have explored the potential of representation-level model adaptation. Specifically, RED \cite{ReE:conf/acl/WuLWLLLZZZ024} performs direct representation editing by applying element-wise scaling and shifting to the full hidden representations in LLMs \cite{LLM:journals/corr/abs-2302-13971}. ReFT \cite{Reft:conf/nips/WuAWGJMP24} further controls partial representations within specific subspaces using low-rank projection matrices with orthogonal rows and linear projectors. These methods show that directly editing internal representations can effectively regulate model behavior. Extending from unimodal models to multimodal settings, recent works \cite{Edit:conf/iclr/0003LTLRDRHLW025,ReE:journals/corr/abs-2604-07965} begin to investigate representation editing in multimodal models. MRT \cite{Edit:conf/iclr/0003LTLRDRHLW025} transfers representation editors from LLMs to MLLMs \cite{MLLM:conf/nips/LiuLWL23a}. However, these methods are mainly designed for LLMs or MLLMs parameter-efficient fine-tuning (PEFT) \cite{PEFT:conf/icml/HoulsbyGJMLGAG19,PEFT:conf/iclr/HuSWALWWC22}, and their effectiveness in multimodal TTA remains largely unexplored. Therefore, our approach fills this gap by layer-wise editing intermediate representations at test time, enabling representation-level adaptation for multimodal TTA.

\begin{figure*}[t]
\centering
\includegraphics[width=0.77\textwidth]{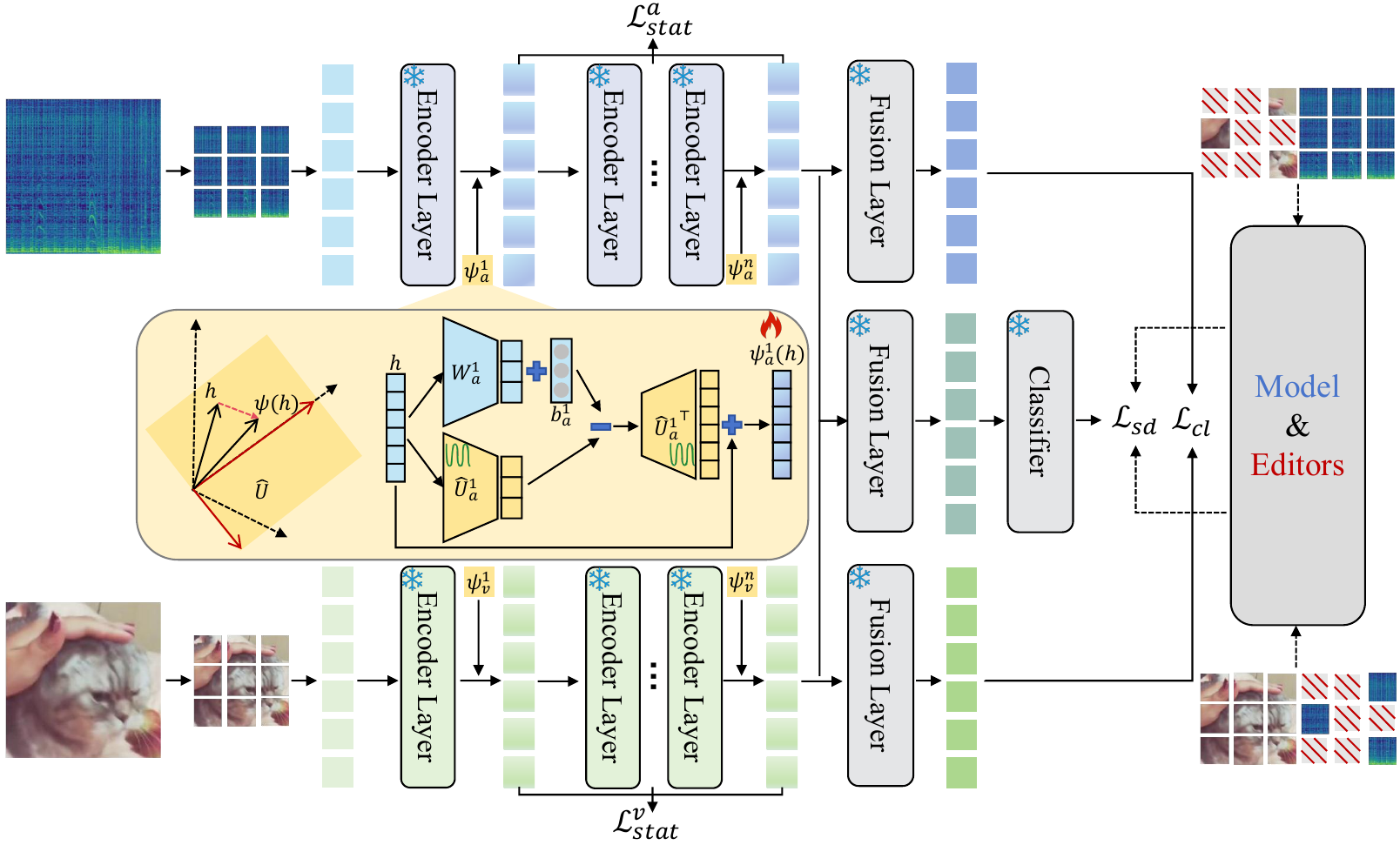}
\caption{Overview of FIRE. The entire model remains frozen, with only editors being updated. The editor aims to identify an editing vector that modifies the representations only within the linear subspace spanned by the rows of $\hat{\U}$. ``Model \& Editors'' is the same pretrained model and editors. In particular, we illustrate the representation editing process in a 3-dimensional space.}
\label{fig:overview}
\end{figure*}

\section{Methodology}
In this section, we introduce FIRE, a novel multimodal TTA method that effectively aligns the target and source distributions. Overview of our approach is provided in Figure \ref{fig:overview}.

\begin{algorithm}[t]
\caption{Algorithm of Fourier Representation Editor}
\label{alg:algorithm}
\textbf{Input:} The intermediate representation
$\h$ and editor parameters $\{\U, \W, \b\}$.\\
\textbf{Output:} The edited representation.
\begin{algorithmic}[1]
\STATE Compute the FFT of $\U$ using Equation \ref{eq:U_fft}; \\
\STATE Compute the enhanced subspace $\hat{\U}$ using Equation \ref{eq:U^}; \\
\STATE Compute the subspace representation $\h'=\W \h + \b$; \\
\STATE Compute the subspace correction $\Delta \h = \hat{\U}^{\top}(\h'-\hat{\U} \h)$;\\
\STATE Obtain edited representation $\psi(\h)=\h+\Delta \h$; \\
\STATE \textbf{Return} representation $\psi(\h)$.
\end{algorithmic}
\end{algorithm}

\subsection{Problem Statement} \label{sec:problem}
Without loss of generality, we take two modalities, i.e., audio and video, as a showcase for clarity of presentation. For the sake of simplicity, we use superscript $r$ to indicate the module corresponding to a specific modality in this section, where $r \in \{ a, v\}$. For clarity, we use $F_{\Theta}=\{ f_{\Theta^a}, f_{\Theta^v}, f_{\Theta^m}, f_{\Theta^c}\}$ to denote the off-the-shelf model that trained on the source domain $\DM_{\SM}=\{ (\x^{a}_i, \x^{v}_i, y_i) \}^{N_{\SM}}_{i=1}$, where $f_{\Theta^a}$ and $f_{\Theta^v}$ are the specific encoders for modality $a$ and $v$, $f_{\Theta^m}$ and $f_{\Theta^c}$ are the fusion layer and the following classifier head. Assume that the target domain $\DM_{\TM}=\{ (\x^{a}_j, \x^{v}_j) \}^{N_{\TM}}_{j=1}$, where the target data follow the distribution $P_{\TM}(\x)$. During training, the model $F_{\Theta}$ would fit the distribution of source data, i.e., $P_{\SM}(\x)$. As a result, in the inference stage, when a domain shift exists between the source and target distributions, i.e., $P_{\SM}(\x) \neq P_{\TM}(\x)$, the performance of $F_{\Theta}$ deteriorates significantly. The goal of TTA is to adapt $F_{\Theta}$ online to the target domain. 

\subsection{Fourier Representation Editor} \label{sec:FIRE}
Since multiple studies \cite{linear:smolensky1986neural,linear:conf/acl/LasriPLPC22,linear:journals/corr/abs-2307-15054} have shown that human-interpretable concepts are encoded linearly and that linear subspace interventions can correspond to specific semantic attributes, we facilitate representation editing by targeting linear subspaces. The final prediction of a multimodal model is determined by the combined contributions of multiple individual modalities, which naturally enables multimodal TTA to be formulated as a set of unimodal TTA problems. We apply representation editors after each transformer layer of the unimodal encoders and jointly optimize all unimodal representations. Notably, during adaptation, only these editors are updated, while the entire model remains completely frozen.

\noindent\textbf{Representation Editor. }Inspired by the linear representation hypothesis \cite{Reft:conf/nips/WuAWGJMP24}, we introduce a representation editor to edit semantically rich intermediate representations. The editor updates the original representation $\h$ within a specific subspace to reflect the desired intervention obtained from a linear projection $\W \h + \b$, where both $\W$ and $\b$ are learnable parameters. Then, the editing operation is confined to the subspace spanned by the rows of a low-rank matrix $\U$ with orthonormal rows. We employ QR decomposition to maintain the row orthonormality of $\U$. Therefore, the editor only adjusts the aspects of the representation associated with the target attributes while preserving the remaining semantic information. The editor can be formed as:
\begin{align}
\Psi(\h)=\h+\U^{\top}(\W \h + \b - \U \h), \label{eq:editor}
\end{align}
where $\U \in \RB^{d_m\times d_n}$ and $\W \in \RB^{d_m\times d_n}$ are low-rank matrices. $d_m$ is the rank of the subspace and $d_n$ is the hidden dimensionality, i.e., $d_m \ll d_n$. Specifically, $\U\h$ projects the original representation $\h$ onto the subspace spanned by $\U$, while $\W \h + \b$ provides the target values within that subspace via a linear projection from $\h$. The difference $\W \h + \b - \U \h$ computes the necessary intervention, which is subsequently mapped back into the original space via $\U^{\top}$. By adding this intervention to $\h$, we obtain the edited representation $\Psi(\h)$ that incorporates the desired modifications while maintaining the components orthogonal to the subspace $\U$.

\noindent\textbf{Fourier-Enhanced Subspace. }The FFT is a powerful algorithm for computing the Discrete Fourier Transform (DFT). Given a sequence $\{ \z_m \}$ where $m$ is a member of the interval $m \in \left [0, M-1 \right ]$, the DFT is defined as:
\begin{align}
\FM(\z)=\ZM_k=\sum_{m=0}^{M-1} \z_m e^{-i2\pi\frac{k}{M}m}, 0\leq k \leq M-1.\label{eq:fft}
\end{align}

Recent parameter-efficient fine-tuning studies \cite{PEFT:conf/icml/GaoWCLWC024,PEFT:conf/nips/Zeng0WWGHWL24,PEFT:conf/nips/BorseKPBGPG0HP24} demonstrate that Fourier transforms can effectively bridge the gap between pretraining and fine-tuning datasets by incorporating additional information. Motivated by this, we explore integrating FFT operations into the representation editor to support on-the-fly adjustments during TTA. To this end, we introduce the FFT to impose global coupling across dimensions in the low-rank matrix $\U$, thereby enriching editable directions and enabling more stable adaptation of the representation editor. Specifically, we first apply an FFT to the low-rank matrix $\U$ along the hidden dimensions:
\begin{align}
\U_{\FM}=\mathfrak{R}(\FM(\U)).\label{eq:U_fft}
\end{align}
To maintain consistency with the original pretrained representation space, we use the FFT only to construct the editing basis and retain its real component $\mathfrak{R}$, ensuring that the final editing basis and representation updates remain in the real-valued space. We then enhance the subspace by combining the editing directions in the original space with the Fourier-domain editing directions, formulated as:
\begin{align}
\hat{\U}=\operatorname{qf}((\U + \U_{\FM})^{\top})^{\top},\label{eq:U^}
\end{align}
where $\operatorname{qf}(\cdot)$ denotes the orthogonal factor obtained from the QR decomposition. 

To improve efficiency, we adopt the Cooley-Tukey FFT algorithm \cite{FFT:cooley1965algorithm}, following common practice \cite{FFT:conf/naacl/Lee-ThorpAEO22}.
The final Fourier representation editor is defined as:
\begin{align}
\psi(\h)=\h+\hat{\U}^{\top}(\W \h + \b - \hat{\U} \h).\label{eq:editor_FFT}
\end{align}
Therefore, the representation editing operation is constrained to the Fourier-enhanced low-rank orthogonal subspace, enabling the method to introduce more diverse editing directions and improve stability without adding extra parameters. To ensure stable TTA, we initialize the representation editor in a no-op scheme, unlike common methods \cite{Reft:conf/nips/WuAWGJMP24,Edit:conf/iclr/0003LTLRDRHLW025}. Our algorithm is summarized in Algorithm \ref{alg:algorithm}. Moreover, we provide a theoretical analysis of the stability and effectiveness of the self-training subspace in the appendix.

\begin{table}[t]
\centering
\caption{Comparison with SOTA methods on Kinetics50-C (top) and VGGSound-C (bottom) benchmarks with corrupted video modality (severity level 5).}
\label{tab:video}
\setlength{\tabcolsep}{0.85mm}
\scriptsize
\begin{tabular}{lcccccccccccccccc}
\toprule
 &
  \multicolumn{3}{c}{Noise} &
  \multicolumn{4}{c}{Blur} &
  \multicolumn{4}{c}{Weather} &
  \multicolumn{4}{c}{Digital} &
   \\ \cmidrule(lr){2-4}\cmidrule(lr){5-8}\cmidrule(lr){9-12}\cmidrule(lr){13-16}
Method &
  Gauss. &
  Shot &
  Impul. &
  Defoc. &
  Glass &
  Motion &
  Zoom &
  Snow &
  Frost &
  Fog &
  Bright. &
  Contr. &
  Elast. &
  Pixel. &
  Jpeg &
  \cellcolor{blue!20}Avg. \\ \midrule
Source &
  48.2 &
  50.0 &
  49.2 &
  67.7 &
  61.6 &
  70.6 &
  66.1 &
  60.9 &
  60.7 &
  44.7 &
  75.9 &
  51.8 &
  65.5 &
  68.7 &
  66.1 &
  \cellcolor{blue!20}60.5 \\
$\bullet$ Tent$_{\mathrm{ICLR21}}$ &
  48.2 &
  49.8 &
  48.7 &
  67.7 &
  62.1 &
  70.8 &
  67.2 &
  61.8 &
  61.4 &
  33.7 &
  76.0 &
  51.2 &
  66.6 &
  69.6 &
  66.9 &
  \cellcolor{blue!20}60.1 \\
$\bullet$ T3A$_{\mathrm{NeurIPS21}}$ &
  47.8 &
  49.0 &
  48.2 &
  62.8 &
  59.3 &
  65.5 &
  62.0 &
  60.2 &
  57.5 &
  50.7 &
  69.8 &
  47.6 &
  62.8 &
  65.7 &
  64.5 &
  \cellcolor{blue!20}58.2 \\
$\bullet$ EATA$_{\mathrm{ICML22}}$ &
  48.7 &
  50.4 &
  49.6 &
  67.8 &
  63.2 &
  70.8 &
  67.5 &
  62.5 &
  62.5 &
  47.9 &
  76.1 &
  52.2 &
  66.9 &
  69.7 &
  67.4 &
  \cellcolor{blue!20}61.5 \\
$\bullet$ SAR$_{\mathrm{ICLR23}}$ &
  48.5 &
  50.2 &
  49.2 &
  67.8 &
  63.8 &
  70.9 &
  67.9 &
  63.1 &
  62.7 &
  38.7 &
  76.1 &
  52.2 &
  67.1 &
  69.8 &
  67.4 &
  \cellcolor{blue!20}61.0 \\
$\bullet$ DeYO$_{\mathrm{ICLR24}}$ &
  48.6 &
  50.2 &
  49.4 &
  67.9 &
  62.6 &
  70.9 &
  67.4 &
  62.5 &
  62.3 &
  40.4 &
  76.1 &
  52.2 &
  66.8 &
  69.8 &
  67.3 &
  \cellcolor{blue!20}61.0 \\
$\bullet$ FOA$_{\mathrm{ICML24}}$ &
  49.2 &
  50.8 &
  49.7 &
  66.0 &
  65.5 &
  69.8 &
  67.4 &
  62.8 &
  65.7 &
  60.3 &
  74.9 &
  51.9 &
  69.5 &
  68.8 &
  68.0 &
  \cellcolor{blue!20}62.7 \\
$\bullet$ READ$_{\mathrm{ICLR24}}$ &
  50.7 &
  52.2 &
  51.4 &
  67.9 &
  65.3 &
  71.1 &
  68.7 &
  64.0 &
  65.8 &
  56.3 &
  76.3 &
  53.6 &
  68.7 &
  70.0 &
  68.6 &
  \cellcolor{blue!20}63.4 \\
$\bullet$ ABPEM$_{\mathrm{AAAI25}}$ &
  52.1 &
  53.1 &
  52.8 &
  \underline{69.0} &
  65.6 &
  71.8 &
  68.8 &
  64.1 &
  65.7 &
  57.9 &
  76.6 &
  54.3 &
  69.2 &
  71.1 &
  69.2 &
  \cellcolor{blue!20}64.1 \\
$\bullet$ SuMi$_{\mathrm{ICLR25}}$ &
  50.1 &
  50.7 &
  50.4 &
  68.2 &
  65.6 &
  72.2 &
  \underline{69.7} &
  65.7 &
  67.0 &
  56.5 &
  \bf{77.1} &
  55.2 &
  69.3 &
  71.2 &
  68.9 &
  \cellcolor{blue!20}63.9 \\
$\bullet$ TSA$_{\mathrm{ICML25}}$ &
  52.6 &
  52.3 &
  52.0 &
  68.7 &
  \underline{68.0} &
  70.7 &
  68.8 &
  65.2 &
  66.6 &
  \underline{64.3} &
  74.6 &
  \underline{57.4} &
  70.5 &
  69.0 &
  66.2 &
  \cellcolor{blue!20}64.5 \\
$\bullet$ PTA$_{\mathrm{NeurIPS25}}$ &
  52.9 &
  53.0 &
  52.5 &
  68.5 &
  66.9 &
  \underline{72.4} &
  69.1 &
  64.9 &
  \underline{67.1} &
  64.1 &
  75.9 &
  55.1 &
  69.2 &
  70.2 &
  68.4 &
  \cellcolor{blue!20}64.7 \\
$\bullet$ BriMPR$_{\mathrm{AAAI26}}$ &
  \underline{55.3} &
  \underline{56.1} &
  \underline{56.7} &
  67.8 &
  67.9 &
  70.6 &
  68.8 &
  65.9 &
  66.2 &
  64.1 &
  76.2 &
  56.3 &
  \underline{72.0} &
  \underline{73.7} &
  \underline{70.5} &
  \cellcolor{blue!20}\underline{65.9} \\
  \rowcolor{green!20}
$\bullet$ FIRE &
  \bf{59.3} &
  \bf{59.8} &
  \bf{59.9} &
  \bf{70.0} &
  \bf{71.1} &
  \bf{73.4} &
  \bf{71.8} &
  \bf{68.7} &
  \bf{69.3} &
  \bf{67.3} &
  \underline{76.6} &
  \bf{61.4} &
  \bf{74.2} &
  \bf{75.1} &
  \bf{71.8} &
  \bf{68.6} \\ \midrule
Source &
  52.9 &
  53.0 &
  53.1 &
  57.2 &
  57.2 &
  58.5 &
  57.5 &
  56.5 &
  57.1 &
  55.6 &
  59.2 &
  53.7 &
  57.1 &
  56.4 &
  57.3 &
  \cellcolor{blue!20}56.2 \\
$\bullet$ Tent$_{\mathrm{ICLR21}}$ &
  53.2 &
  53.3 &
  53.3 &
  56.8 &
  56.6 &
  57.9 &
  57.2 &
  55.9 &
  56.6 &
  56.5 &
  58.5 &
  53.9 &
  57.5 &
  56.8 &
  56.9 &
  \cellcolor{blue!20}56.1 \\
$\bullet$ T3A$_{\mathrm{NeurIPS21}}$ &
  53.2 &
  53.1 &
  53.3 &
  57.2 &
  57.3 &
  58.5 &
  58.0 &
  56.6 &
  57.2 &
  56.7 &
  59.4 &
  54.3 &
  57.8 &
  56.9 &
  57.5 &
  \cellcolor{blue!20}56.5 \\
$\bullet$ EATA$_{\mathrm{ICML22}}$ &
  53.4 &
  53.5 &
  53.5 &
  57.0 &
  57.0 &
  58.3 &
  57.7 &
  56.3 &
  57.0 &
  56.8 &
  59.1 &
  54.2 &
  57.9 &
  57.2 &
  57.2 &
  \cellcolor{blue!20}56.4 \\
$\bullet$ SAR$_{\mathrm{ICLR23}}$ &
  53.3 &
  53.3 &
  53.3 &
  56.4 &
  56.5 &
  57.9 &
  57.3 &
  55.6 &
  56.4 &
  56.3 &
  58.8 &
  53.7 &
  57.8 &
  56.9 &
  57.0 &
  \cellcolor{blue!20}56.0 \\
$\bullet$ DeYO$_{\mathrm{ICLR24}}$ &
  53.3 &
  53.4 &
  53.4 &
  56.7 &
  56.7 &
  58.0 &
  57.3 &
  56.0 &
  56.8 &
  56.4 &
  58.7 &
  53.9 &
  57.7 &
  57.0 &
  57.0 &
  \cellcolor{blue!20}56.2 \\
$\bullet$ FOA$_{\mathrm{ICML24}}$ &
  52.7 &
  52.7 &
  52.7 &
  53.2 &
  53.6 &
  53.6 &
  53.8 &
  53.4 &
  53.4 &
  53.3 &
  55.6 &
  52.5 &
  55.3 &
  53.7 &
  54.4 &
  \cellcolor{blue!20}53.6 \\
$\bullet$ READ$_{\mathrm{ICLR24}}$ &
  53.8 &
  54.0 &
  53.8 &
  58.0 &
  57.9 &
  59.2 &
  \underline{58.7} &
  57.1 &
  \bf{58.2} &
  50.0 &
  60.0 &
  \underline{55.2} &
  58.5 &
  57.7 &
  58.2 &
  \cellcolor{blue!20}56.7 \\
$\bullet$ ABPEM$_{\mathrm{AAAI25}}$ &
  46.5 &
  46.7 &
  46.5 &
  54.2 &
  55.1 &
  56.4 &
  55.2 &
  51.3 &
  53.2 &
  52.1 &
  56.6 &
  52.1 &
  54.4 &
  51.7 &
  54.7 &
  \cellcolor{blue!20}52.4 \\
$\bullet$ SuMi$_{\mathrm{ICLR25}}$ &
  54.0 &
  54.3 &
  53.8 &
  \underline{58.2} &
  58.4 &
  \underline{59.4} &
  \underline{58.7} &
  \underline{57.5} &
  \bf{58.2} &
  57.6 &
  59.4 &
  54.8 &
  59.0 &
  57.5 &
  58.2 &
  \cellcolor{blue!20}57.3 \\
$\bullet$ TSA$_{\mathrm{ICML25}}$ &
  53.9 &
  53.9 &
  54.0 &
  57.6 &
  58.0 &
  59.0 &
  58.6 &
  56.9 &
  57.0 &
  56.6 &
  59.8 &
  54.7 &
  58.6 &
  56.7 &
  57.9 &
  \cellcolor{blue!20}56.9 \\
$\bullet$ PTA$_{\mathrm{NeurIPS25}}$ &
  48.3 &
  48.2 &
  48.1 &
  55.9 &
  56.5 &
  57.8 &
  56.4 &
  53.9 &
  54.2 &
  53.2 &
  58.3 &
  52.0 &
  55.8 &
  53.6 &
  56.3 &
  \cellcolor{blue!20}53.9 \\
$\bullet$ BriMPR$_{\mathrm{AAAI26}}$ &
  \underline{54.9} &
  \underline{55.0} &
  \underline{55.0} &
  57.9 &
  \underline{58.5} &
  58.9 &
  \underline{58.7} &
  \underline{57.5} &
  \underline{58.0} &
  \underline{58.5} &
  \underline{60.3} &
  54.5 &
  \underline{59.7} &
  \underline{59.3} &
  \underline{59.0} &
  \cellcolor{blue!20}\underline{57.7} \\
\rowcolor{green!20} 
$\bullet$ FIRE &
  \bf{55.5} &
  \bf{55.5} &
  \bf{55.5} &
  \bf{58.4} &
  \bf{59.0} &
  \bf{59.6} &
  \bf{59.3} &
  \bf{57.8} &
  \bf{58.2} &
  \bf{58.7} &
  \bf{60.5} &
  \bf{56.2} &
  \bf{60.1} &
  \bf{59.7} &
  \bf{59.4} &
  \bf{58.2} \\ \bottomrule
\end{tabular}
\end{table}

\subsection{Adaptation through Self-Training} \label{sec:loss}
Due to the absence of supervision at test time, we optimize the representation editors with self-training objectives. Inspired by related tasks \cite{FOA:conf/icml/NiuMCWZ24,CLIP:conf/icml/RadfordKHRGASAM21,Mask:conf/iclr/RamazanovaPGA25,BriMPR:conf/aaai/LiF26}, we constrain editor adaptation at multiple levels, including cross-modal semantic alignment, source-target statistical alignment, and asymmetric prediction consistency.

\noindent\textbf{Cross-modal Semantic Alignment Loss. }The cross-modal semantic alignment loss aims to realign semantic relationships across modalities. Cross-modal contrastive learning \cite{CLIP:conf/icml/RadfordKHRGASAM21} is a key paradigm for multimodal learning and can enhance semantic alignment across modalities. In multimodal TTA, we introduce contrastive loss to help the model realign the semantics of shifted and unshifted modalities. Specifically, given final unimodal representations $\tilde{\h}_a$ and $\tilde{\h}_v$, the cross-modal semantic alignment loss is defined as:
\begin{equation}
\begin{aligned}
\LM_{cl}=&-\frac{1}{2B}\sum_{j=1}^{B}(\log{\frac{e^{sim(\tilde{\h}_{a,j},\tilde{\h}_{v,j})/\tau}}{\sum_{k=1}^{B}e^{sim(\tilde{\h}_{a,j},\tilde{\h}_{v,k})/\tau}}} \\ 
& + \log{\frac{e^{sim(\tilde{\h}_{a,j},\tilde{\h}_{v,j})/\tau}}{\sum_{k=1}^{B}e^{sim(\tilde{\h}_{a,k},\tilde{\h}_{v,j})/\tau}}}),\label{loss:cl}
\end{aligned}
\end{equation}
where $B$ denotes the mini-batch size, $sim(\cdot,\cdot)$ denotes the cosine similarity, and $\tau$ denotes the temperature hyperparameter set to its default value..

\noindent\textbf{Source-target Statistical Alignment Loss. }Source-target statistical alignment loss is an effective distribution calibration objective, which calibrates target representations by aligning their layer-wise statistics with those of the source domain. Specifically, we offline compute the mean $\{\mu_{l}^{\SM,r} \}_{l=1}^{n}$ and standard deviation $\{\sigma_{l}^{\SM,r} \}_{l=1}^{n}$ of the source-domain unimodal representations at each layer after model pretraining and before TTA, where $l$ denotes the layer index. During TTA, given a mini-batch of test samples, we obtain the edited representation sequence from the $l$-th layer of a specific modality encoder. We then compute the corresponding unimodal representation statistics of the test samples, $\{\mu_{l}^{\TM,r} \}_{l=1}^{n}$ and $\{\sigma_{l}^{\TM,r} \}_{l=1}^{n}$, and align them with the precomputed source statistics. This enables target representation calibration toward the source distribution, requiring only cached source statistics rather than access to the source domain data. Formally, the loss is defined by:
\begin{equation}
\begin{aligned}
\LM_{stat}(\mu^{\SM},\sigma^{\SM},\mu^{\TM},\sigma^{\TM})=\sum_{r\in \{ a,v\}}\LM_{stat}^{r}(\mu^{\SM},\sigma^{\SM},\mu^{\TM},\sigma^{\TM})\\
=\sum_{r\in \{ a,v\}} \frac{1}{n} \sum_{l=1}^{n}(\|\mu_{l}^{\TM,r}-\mu_{l}^{\SM,r}\|_{2}+\|\sigma_{l}^{\TM,r}-\sigma_{l}^{\SM,r}\|_{2}),\label{loss:stat}
\end{aligned}
\end{equation}
where $\|\cdot \|_2$ denotes the Euclidean norm. Thus, statistics matching provides a lightweight proxy for source-target representation alignment without assuming a specific representation distribution.

\noindent\textbf{Asymmetric Prediction Consistency Loss. }To improve model predictions, we introduce an asymmetric prediction consistency loss. Ideally, the model $F_{\Theta}$ should produce consistent predictions under both complete and incomplete modalities, making it robust to distribution shifts. Specifically, for each test sample, we randomly mask a portion of information from one modality (75\% by default) while keeping the other modality unchanged, yielding two sets of incomplete sample pairs. These two sets are then fed into the model and layer-wise calibrated by the representation editor at each layer to obtain the corresponding predictions. Finally, we use self-distillation to encourage the model to be insensitive to the specific modality source present during testing. The loss is formulated as:
\begin{align}
\LM_{CE}^{r}(\p^{r},\p)=\frac{1}{B}\sum_{j=1}^{B}\ell(\p^{r}_{j},\p_{j})=-\frac{1}{B}\sum_{j=1}^{B}\p_j^\top\log(\p^{r}_{j}),\label{loss:msd}
\end{align}
where $\p^{r}$ denotes the prediction after masking modality $r$, $\p$ denotes the prediction with complete modalities, and $\ell$ is the cross entropy loss. Soft predictions as consistency targets avoid committing to potentially incorrect hard pseudo-labels, thereby reducing the risk of confirmation bias.

In multimodal TTA, different modalities contribute unequally to adaptation \cite{READ:conf/iclr/Yang0Z0024}. Treating all incomplete sample pairs equally may degrade unshifted modalities. Therefore, when one modality is shifted, we encourage masking the other unshifted modality, forcing the representation editor to learn target-to-source alignment. Based on this insight, asymmetric prediction consistency loss weights the losses of all incomplete sample pairs according to the discrepancy between the current batch statistics of each modality and the corresponding source domain statistics:
\begin{align}
\LM_{pc}(\p^{a},\p^{v},\p)=\frac{\LM_{stat}^{v}}{\LM_{stat}}\LM_{CE}^{a}(\p^{a},\p)+\frac{\LM_{stat}^{a}}{\LM_{stat}}\LM_{CE}^{v}(\p^{v},\p).\label{loss:sd}
\end{align}

\begin{table}[t]
\centering
\caption{Comparison with SOTA methods on Kinetics50-C (left) and VGGSound-C (right) benchmarks with corrupted audio modality (severity level 5).}
\label{tab:audio}
\setlength{\tabcolsep}{0.85mm}
\scriptsize
\begin{tabular}{lcccccccccccccc}
\toprule
 &
  \multicolumn{3}{c}{Noise} &
  \multicolumn{3}{c}{Weather} &
   &
  \multicolumn{3}{c}{Noise} &
  \multicolumn{3}{c}{Weather} &
   \\ \cmidrule(lr){2-4} \cmidrule(lr){5-7} \cmidrule(lr){9-11} \cmidrule(lr){12-14}
Method &
  Gauss. &
  Traff. &
  Crowd &
  Rain &
  Thund. &
  Wind &
  \cellcolor{blue!20}Avg. &
  Gauss. &
  Traff. &
  Crowd &
  Rain &
  Thund. &
  Wind &
  \cellcolor{blue!20}Avg. \\ \midrule
Source &
  74.3 &
  65.3 &
  68.0 &
  70.3 &
  68.0 &
  70.5 &
  \cellcolor{blue!20}69.4 &
  37.3 &
  21.2 &
  16.9 &
  21.8 &
  27.3 &
  25.7 &
  \cellcolor{blue!20}25.0 \\
$\bullet$ Tent$_{\mathrm{ICLR21}}$ &
  74.6 &
  67.4 &
  69.5 &
  70.8 &
  67.6 &
  71.2 &
  \cellcolor{blue!20}70.2 &
  10.8 &
  2.8 &
  1.8 &
  2.9 &
  5.6 &
  3.9 &
  \cellcolor{blue!20}4.6 \\
$\bullet$ T3A$_{\mathrm{NeurIPS21}}$ &
  71.4 &
  63.0 &
  65.4 &
  66.8 &
  66.1 &
  67.8 &
  \cellcolor{blue!20}66.8 &
  39.8 &
  28.2 &
  24.8 &
  28.6 &
  34.7 &
  30.6 &
  \cellcolor{blue!20}31.1 \\
$\bullet$ EATA$_{\mathrm{ICML22}}$ &
  74.6 &
  67.3 &
  69.4 &
  70.8 &
  69.8 &
  71.0 &
  \cellcolor{blue!20}70.5 &
  40.2 &
  30.0 &
  27.8 &
  29.7 &
  36.5 &
  32.2 &
  \cellcolor{blue!20}32.7 \\
$\bullet$ SAR$_{\mathrm{ICLR23}}$ &
  74.6 &
  67.0 &
  69.2 &
  70.9 &
  69.5 &
  70.9 &
  \cellcolor{blue!20}70.3 &
  30.4 &
  5.5 &
  8.0 &
  9.3 &
  32.5 &
  17.2 &
  \cellcolor{blue!20}17.1 \\
$\bullet$ DeYO$_{\mathrm{ICLR24}}$ &
  74.6 &
  67.0 &
  69.3 &
  70.8 &
  69.0 &
  71.0 &
  \cellcolor{blue!20}70.3 &
  22.9 &
  4.9 &
  15.8 &
  4.9 &
  16.5 &
  20.0 &
  \cellcolor{blue!20}14.2 \\
$\bullet$ FOA$_{\mathrm{ICML24}}$ &
  73.8 &
  \underline{70.0} &
  70.5 &
  71.0 &
  73.0 &
  71.2 &
  \cellcolor{blue!20}71.6 &
  31.5 &
  26.2 &
  23.7 &
  31.0 &
  34.2 &
  26.7 &
  \cellcolor{blue!20}28.9 \\
$\bullet$ READ$_{\mathrm{ICLR24}}$ &
  \underline{74.8} &
  69.2 &
  69.9 &
  71.4 &
  72.4 &
  71.0 &
  \cellcolor{blue!20}71.5 &
  39.9 &
  29.4 &
  26.8 &
  30.8 &
  36.8 &
  30.7 &
  \cellcolor{blue!20}32.4 \\
$\bullet$ ABPEM$_{\mathrm{AAAI25}}$ &
  74.7 &
  68.5 &
  70.3 &
  \underline{71.7} &
  72.3 &
  71.2 &
  \cellcolor{blue!20}71.4 &
  38.5 &
  27.6 &
  25.2 &
  26.5 &
  32.7 &
  26.5 &
  \cellcolor{blue!20}29.5 \\
$\bullet$ SuMi$_{\mathrm{ICLR25}}$ &
  \bf{75.1} &
  68.9 &
  70.6 &
  71.6 &
  72.8 &
  \bf{72.1} &
  \cellcolor{blue!20}71.9 &
  \bf{41.9} &
  26.3 &
  27.9 &
  31.6 &
  37.1 &
  34.1 &
  \cellcolor{blue!20}33.2 \\
$\bullet$ TSA$_{\mathrm{ICML25}}$ &
  74.5 &
  69.6 &
  70.5 &
  71.4 &
  72.0 &
  71.0 &
  \cellcolor{blue!20}71.5 &
  \underline{41.5} &
  31.8 &
  30.9 &
  \bf{32.6} &
  38.9 &
  32.6 &
  \cellcolor{blue!20}34.7 \\
$\bullet$ PTA$_{\mathrm{NeurIPS25}}$ &
  73.5 &
  \underline{70.0} &
  70.4 &
  70.3 &
  \underline{73.4} &
  70.8 &
  \cellcolor{blue!20}71.4 &
  40.5 &
  30.8 &
  31.8 &
  32.0 &
  38.6 &
  31.7 &
  \cellcolor{blue!20}34.2 \\ 
$\bullet$ BriMPR$_{\mathrm{AAAI26}}$ &
  \underline{74.8} &
  69.6 &
  \underline{71.7} &
  71.5 &
  72.4 &
  \underline{72.0} &
  \cellcolor{blue!20}\underline{72.0} &
  39.3 &
  \underline{35.0} &
  \underline{36.7} &
  \underline{32.5} &
  \underline{41.0} &
  \underline{34.6} &
  \cellcolor{blue!20}\underline{36.5} \\
\rowcolor{green!20}
$\bullet$ FIRE &
  \underline{74.8} &
  \bf{72.4} &
  \bf{73.5} &
  \bf{72.4} &
  \bf{75.2} &
  \bf{72.1} &
  \bf{73.4} &
  39.9 &
  \bf{35.5} &
  \bf{37.2} &
  32.4 &
  \bf{41.6} &
  \bf{35.1} &
  \bf{36.9} \\ \bottomrule
\end{tabular}
\end{table}

\noindent\textbf{Multi-level Adaptation Objective. }The final multi-level adaptation objective combines these three losses:
\begin{align}
\LM_{all}=\LM_{cl}+\LM_{stat}+\LM_{pc}.\label{loss:final}
\end{align}
By minimizing this final combined loss, we can optimize the representation editors to achieve target-to-source alignment at test time. 

\begin{table}[t]
\centering
\caption{Comparison with SOTA methods on Kinetics50-C benchmarks for multimodal domain shifts (severity level 5).}
\label{tab:ks50_both}
\setlength{\tabcolsep}{0.85mm}
\scriptsize
\begin{tabular}{lcccccccccccccccc}
\toprule
 &
  \multicolumn{3}{c}{Noise} &
  \multicolumn{4}{c}{Blur} &
  \multicolumn{4}{c}{Weather} &
  \multicolumn{4}{c}{Digital} &
   \\ \cmidrule(lr){2-4}\cmidrule(lr){5-8}\cmidrule(lr){9-12}\cmidrule(lr){13-16}
Method &
  Gauss. &
  Shot &
  Impul. &
  Defoc. &
  Glass &
  Motion &
  Zoom &
  Snow &
  Frost &
  Fog &
  Bright. &
  Contr. &
  Elast. &
  Pixel. &
  Jpeg &
  \cellcolor{blue!20}Avg. \\ \midrule
Source &
  13.1 &
  14.1 &
  13.3 &
  37.2 &
  37.4 &
  45.3 &
  41.8 &
  29.4 &
  32.6 &
  20.4 &
  55.2 &
  18.3 &
  42.5 &
  38.8 &
  37.8 &
  \cellcolor{blue!20}31.8 \\
$\bullet$ Tent$_{\mathrm{ICLR21}}$ &
  9.1 &
  9.7 &
  9.1 &
  32.5 &
  34.1 &
  43.5 &
  40.2 &
  23.2 &
  28.3 &
  13.2 &
  55.1 &
  13.7 &
  40.7 &
  34.7 &
  35.0 &
  \cellcolor{blue!20}28.1 \\
$\bullet$ T3A$_{\mathrm{NeurIPS21}}$ &
  12.3 &
  13.8 &
  12.7 &
  34.7 &
  35.5 &
  42.0 &
  39.3 &
  29.2 &
  28.7 &
  20.9 &
  51.9 &
  17.1 &
  40.6 &
  37.5 &
  37.4 &
  \cellcolor{blue!20}30.2 \\
$\bullet$ EATA$_{\mathrm{ICML22}}$ &
  12.9 &
  14.0 &
  13.1 &
  38.1 &
  38.7 &
  46.9 &
  43.1 &
  30.6 &
  33.0 &
  20.2 &
  56.5 &
  18.2 &
  43.7 &
  40.7 &
  39.0 &
  \cellcolor{blue!20}32.6 \\
$\bullet$ SAR$_{\mathrm{ICLR23}}$ &
  11.8 &
  12.8 &
  11.9 &
  37.4 &
  38.2 &
  46.3 &
  43.1 &
  29.8 &
  33.0 &
  17.4 &
  56.0 &
  16.0 &
  43.5 &
  39.5 &
  38.2 &
  \cellcolor{blue!20}31.7 \\
$\bullet$ DeYO$_{\mathrm{ICLR24}}$ &
  11.0 &
  12.0 &
  11.1 &
  37.0 &
  37.7 &
  46.3 &
  43.2 &
  29.9 &
  33.3 &
  17.9 &
  56.2 &
  17.0 &
  43.7 &
  39.7 &
  38.0 &
  \cellcolor{blue!20}31.6 \\
$\bullet$ FOA$_{\mathrm{ICML24}}$ &
  18.7 &
  20.6 &
  19.3 &
  43.7 &
  45.5 &
  50.2 &
  47.9 &
  38.9 &
  43.7 &
  37.2 &
  \underline{60.5} &
  23.5 &
  52.7 &
  48.9 &
  47.4 &
  \cellcolor{blue!20}39.9 \\
$\bullet$ READ$_{\mathrm{ICLR24}}$ &
  14.5 &
  14.9 &
  14.8 &
  43.8 &
  42.1 &
  51.0 &
  46.5 &
  35.4 &
  38.9 &
  27.6 &
  58.9 &
  22.6 &
  47.1 &
  42.1 &
  38.1 &
  \cellcolor{blue!20}35.9 \\
$\bullet$ ABPEM$_{\mathrm{AAAI25}}$ &
  19.2 &
  20.7 &
  19.7 &
  46.2 &
  44.2 &
  51.9 &
  47.9 &
  38.1 &
  41.1 &
  32.6 &
  59.9 &
  25.3 &
  49.4 & 
  48.8 &
  45.6 &
  \cellcolor{blue!20}39.4 \\
$\bullet$ SuMi$_{\mathrm{ICLR25}}$ &
  12.5 &
  13.6 &
  12.6 &
  37.0 &
  37.9 &
  45.9 &
  42.3 &
  29.3 &
  32.7 &
  19.7 &
  55.7 &
  17.8 &
  42.7 &
  38.3 &
  36.9 &
  \cellcolor{blue!20}31.7 \\
$\bullet$ TSA$_{\mathrm{ICML25}}$ &
  18.1 &
  18.7 &
  17.9 &
  47.1 &
  \underline{47.7} &
  \underline{53.0} &
  \underline{50.1} &
  \underline{41.0} &
  \underline{45.1} &
  \underline{41.6} &
  59.6 &
  \underline{29.5} &
  \underline{53.1} &
  48.4 &
  45.2 &
  \cellcolor{blue!20}41.1 \\
$\bullet$ PTA$_{\mathrm{NeurIPS25}}$ &
  21.9 &
  23.0 &
  22.1 &
  \underline{47.7} &
  45.9 &
  52.6 &
  49.3 &
  40.3 &
  43.6 &
  39.7 &
  60.0 &
  27.3 &
  50.4 &
  50.9 &
  47.6 &
  \cellcolor{blue!20}\underline{41.5} \\
$\bullet$ BriMPR$_{\mathrm{AAAI26}}$ &
  \underline{22.9} &
  \underline{24.2} &
  \underline{24.1} &
  43.6 &
  45.4 &
  49.5 &
  48.2 &
  38.0 &
  40.8 &
  36.8 &
  59.8 &
  27.1 &
  52.8 &
  \underline{52.7} &
  \underline{47.9} &
  \cellcolor{blue!20}40.9 \\
\rowcolor{green!20} 
$\bullet$ FIRE &
  \bf{30.2} &
  \bf{30.8} &
  \bf{31.4} &
  \bf{48.7} &
  \bf{52.2} &
  \bf{54.2} &
  \bf{53.3} &
  \bf{44.5} &
  \bf{47.5} &
  \bf{45.7} &
  \bf{61.9} &
  \bf{35.1} &
  \bf{58.2} &
  \bf{58.5} &
  \bf{53.2} &
  \bf{47.0} \\ \bottomrule
\end{tabular}
\end{table}

\section{Experiments}
\subsection{Experimental Setup}
\noindent\textbf{Datasets. }We evaluate FIRE on the widely used multimodal datasets Kinetics50-C and VGGSound-C \cite{READ:conf/iclr/Yang0Z0024}, which are corrupted versions of Kinetics50 \cite{KS:journals/corr/KayCSZHVVGBNSZ17} and VGGSound \cite{vgg:conf/icassp/ChenXVZ20} datasets, respectively. Specifically, Kinetics50 contains YouTube videos of human actions across 50 categories, while VGGSound contains 309 audio-visual classes covering a wide range of real-world audio-visual events. Following \cite{READ:conf/iclr/Yang0Z0024}, we introduce 6 and 15 types of corruptions for the audio and video modalities, respectively, with each corruption type divided into 5 severity levels to simulate real-world multimodal distribution shifts. For multimodal corruptions, we use combinations of unimodal corruptions, resulting in 90 combinations for each benchmark. In particular, following \cite{READ:conf/iclr/Yang0Z0024,TSA:conf/icml/ChenZHJWF0B25,PTA:wang2026partition,BriMPR:conf/aaai/LiF26}, we select the highest corruption severity level. More details are provided in the appendix.

\begin{wraptable}{r}{0.55\textwidth}
\centering
\caption{Comparison with TTA baseline methods on VGGSound-C for multimodal domain shifts (severity level 5).}
\label{tab:vgg_both}
\setlength{\tabcolsep}{0.85mm}
\scriptsize
\begin{tabular}{lccccccc}
\toprule
& \multicolumn{3}{c}{Noise} & \multicolumn{3}{c}{Weather} & \\ \cmidrule(lr){2-4} \cmidrule(lr){5-7}
Method & Gauss. & Traff. & Crowd & Rain & Thund. & Wind & \cellcolor{blue!20}Avg. \\ \midrule
Source & 17.1 & 6.4 & 5.4 & 6.0 & 13.5 & 8.8 & \cellcolor{blue!20}9.5 \\
$\bullet$ Tent & 3.2 & 0.9 & 0.8 & 0.9 & 2.8 & 1.3 & \cellcolor{blue!20}1.6 \\
$\bullet$ T3A  & 20.6 & 9.2 & 7.8 & 8.6 & 17.4 & 10.9 & \cellcolor{blue!20}12.4 \\
$\bullet$ EATA & 21.5 & 7.7 & 7.1 & 7.3 & 17.3 & 11.9 & \cellcolor{blue!20}12.1 \\
$\bullet$ SAR & 10.7 & 1.8 & 1.6 & 2.3 & 12.8 & 3.1 & \cellcolor{blue!20}5.4 \\
$\bullet$ DeYO & 6.7 & 1.2 & 1.3 & 1.3 & 9.3 & 2.9 & \cellcolor{blue!20}3.8 \\
$\bullet$ FOA & 18.8 & 10.8 & 11.4 & 11.6 & 20.5 & 10.4 & \cellcolor{blue!20}13.9 \\
$\bullet$ READ & 20.1 & 12.5 & 10.7 & 10.5 & 20.5 & 13.4 & \cellcolor{blue!20}14.6 \\
$\bullet$ ABPEM & 21.9 & 13.4 & 12.3 & 10.9 & 20.4 & 12.4 & \cellcolor{blue!20}15.2 \\
$\bullet$ SuMi & 17.0 & 6.8 & 5.7 & 6.2 & 13.4 & 8.8 & \cellcolor{blue!20}9.7 \\
$\bullet$ TSA & 20.5 & 12.6 & 12.2 & 12.2 & 22.0 & 13.3 & \cellcolor{blue!20}15.5 \\
$\bullet$ PTA & \underline{24.3} & 15.3 & 15.3 & 13.7 & 23.8 & 15.0 & \cellcolor{blue!20}17.9 \\
$\bullet$ BriMPR & 23.5 & \underline{18.8} & \underline{21.4} & \underline{15.8} & \underline{26.8} & \underline{18.3} & \cellcolor{blue!20}\underline{20.7} \\
\rowcolor{green!20} 
$\bullet$ FIRE & \bf{25.8} & \bf{19.6} & \bf{22.4} & \bf{16.9} & \bf{27.5} & \bf{19.7} & \bf{22.0} \\ \bottomrule
\end{tabular}
\end{wraptable}

\noindent\textbf{Baselines. }We select two categories of methods for comparison. In detail, unimodal TTA methods: Tent \cite{Tent:conf/iclr/WangSLOD21}, T3A \cite{T3A:conf/nips/IwasawaM21}, EATA \cite{EATA:conf/icml/NiuW0CZZT22}, SAR \cite{SAR:conf/iclr/Niu00WCZT23}, DeYO \cite{DeYO:conf/iclr/0004JLPSHY24}, and FOA \cite{FOA:conf/icml/NiuMCWZ24}. Multimodal TTA methods: READ \cite{READ:conf/iclr/Yang0Z0024}, ABPEM \cite{ABPEM:conf/aaai/Zhao0LHY0025}, SuMi \cite{SuMi:conf/iclr/Guo025}, TSA \cite{TSA:conf/icml/ChenZHJWF0B25}, PTA \cite{PTA:wang2026partition}, and BriMPR \cite{BriMPR:conf/aaai/LiF26}.

\noindent\textbf{Implementation Details. }Following prior work \cite{READ:conf/iclr/Yang0Z0024}, we use the pretrained CAV-MAE \cite{CAVMAE:conf/iclr/GongRLHKKG23} as the source model. For all experiments, we set the learning rate to 1e-4 and employ the Adam \cite{Adam:journals/corr/KingmaB14} optimizer. The batch size is set to 16 for Kinetics50-C and 64 for VGGSound-C. The subspace $\hat{\U}$ is orthogonally reparameterized to preserve its orthogonality during TTA. The matrix $\W$ is initialized with the initial orthogonal basis $\hat{\U}$, while $\b$ is initialized to 0, yielding a no-op initialization to ensure stability. Unless otherwise specified, the rank of the low-rank matrices is set to 6. Notably, larger rank can yield better performance, while increasing the parameter count. For competing methods, we adopt the hyperparameters recommended in respective papers. All experiments are conducted on NVIDIA RTX-3090 GPUs with 5 random seeds. Following prior work \cite{READ:conf/iclr/Yang0Z0024}, we report the mean results over five random seeds.

\subsection{Main Results}
We report the results under unimodal and multimodal shift settings in Table \ref{tab:video}, Table \ref{tab:audio}, Table \ref{tab:ks50_both}, and Table \ref{tab:vgg_both}, where \textbf{bold} and \underline{underlined} denote the best and second-best results, respectively. All results are averaged over 5 random seeds.

\noindent\textbf{Robustness to unimodal shift setting. }In Table \ref{tab:video} and Table \ref{tab:audio}, we evaluate the unimodal shift setting on Kinetics50-C and VGGSound-C datasets, where the audio and video modalities are corrupted, respectively. We can observe that: (1) Unimodal TTA methods exhibit limited effectiveness in the multimodal setting and may even suffer from negative transfer, whereas multimodal TTA methods yield consistent improvements. (2) Our approach outperforms SOTA baselines, achieving the best average performance across datasets.

\noindent\textbf{Robustness to multimodal shift setting. }Table \ref{tab:ks50_both} and Table \ref{tab:vgg_both} present results under the more challenging multimodal shift setting, where online data from both audio and video modalities is corrupted. Each entry in the table denotes the average performance over all corruptions of the other modality under a specific corruption of one modality. We can draw conclusions: (1) Most TTA methods yield limited gains and may even suffer from negative transfer. (2) Our approach achieves substantial improvements in this more complex setting.

\subsection{Further Analysis}
\noindent\textbf{Impact of Rank. }In the representation editor, the number of ranks directly determines the number of trainable parameters and editing directions. To further investigate the impact of rank on TTA performance, we conduct a comprehensive study of the editor rank on Kinetics50-C dataset. Specifically, we consider different editor ranks, ranging from 2 to 10. The results are shown in Table \ref{tab:rank_depth}. We can observe that TTA performance improves as the rank increases. Nevertheless, thanks to the parameter efficiency, even a small rank can achieve competitiveness, relieving us from the need to exhaustively search for the optimal rank.

\begin{table}[t]
\centering
\begin{tabular}{cc}
\begin{minipage}{.48\textwidth}
\centering
\caption{Impact of rank (left) and editing depth (right) on Kinetics50-C dataset.}
\label{tab:rank_depth}
\setlength{\tabcolsep}{0.85mm}
\scriptsize
\begin{tabular}{lccc|lccc}
\toprule
\multicolumn{1}{c}{Rank} & audio & video & both & \multicolumn{1}{c}{Depth} & audio & video & both \\ \midrule
2 & 72.8 & 67.3 & 44.5 &  First Layer & 71.3 & 63.3 & 36.9 \\
4 & 73.3 & 68.2 & 46.3  & First Half & 72.7 & 67.7 & 45.2  \\ 
\cellcolor{green!20}6 & \cellcolor{green!20}73.4 & \cellcolor{green!20}68.6 & \cellcolor{green!20}47.0 & Latter Half & 72.7 & 65.7 & 42.2 \\
8 & \bf 73.5 & 69.0 & 47.5 & Last Layer & 71.2 & 63.3 & 36.5  \\
10 & \bf 73.5 & \bf 69.2 & \bf 48.0 & \cellcolor{green!20}All Layers & \cellcolor{green!20}\bf 73.4 & \cellcolor{green!20}\bf 68.6 & \cellcolor{green!20}\bf 47.0 \\ \bottomrule
\end{tabular}
\end{minipage}
\hfill
\begin{minipage}{.48\textwidth}
\caption{The number of trainable parameters and performance on Kinetics50-C dataset.}
\label{tab:param}
\centering
\setlength{\tabcolsep}{0.85mm}
\scriptsize
\begin{tabular}{llcccc}
\toprule 
\multicolumn{1}{c}{Method}& \multicolumn{1}{c}{layer}& audio & video & both & \multicolumn{1}{c}{Params (M)} \\ \midrule
Tent & Normal & 70.2 & 60.1 & 28.1 & 0.218 \\
READ & Attention & 71.5 & 63.4 & 35.9 & 1.772 \\
TSA & Adapter & 71.5 & 64.5 & 41.1 & 1.180 \\
BriMPR & Prompt & 72.0 & 65.9 & 40.9 & 0.169 \\
\rowcolor{green!20}
FIRE (Rank2) & Editor & \bf 72.8 & \bf 67.3  & \bf 44.5 & \bf 0.068 \\
\bottomrule
\end{tabular}
\end{minipage} 
\end{tabular}
\end{table}

\noindent\textbf{Impact of editing depth. }We examine the influence of editing depth for representation editors on multimodal TTA performance under 5 different settings: the first layer, the last layer, the first half, the latter half, and all layers. The results are reported in Table \ref{tab:rank_depth}. We find that, with the same number of edited layers, shallow intervention performs better than deep, which is consistent with the intuition that representations should be calibrated as early as possible. In addition, editing multiple layers outperforms editing a single layer, and all layers achieve the best performance.

\begin{wraptable}{r}{0.45\textwidth}
\centering
\caption{Ablation study for different components of FIRE.}
\label{tab:ablation}
\setlength{\tabcolsep}{0.85mm}
\scriptsize
\begin{tabular}{lcccccc}
\toprule
& \multicolumn{3}{c}{Kinetics50-C} & \multicolumn{3}{c}{VGGSound-C} \\ \cmidrule(l){2-4} \cmidrule(l){5-7}
\multicolumn{1}{c}{Method}& audio & video & both & audio & video & both \\ \midrule
Source & 69.4 & 60.5 & 31.8 & 25.0 & 56.2 & 9.5\\
w/o FFT & 72.8 & 67.5 & 46.1 & 36.7 & 58.0 & 21.6\\
w/o $\LM_{cl}$ & 72.9 & 67.9 & 46.8 & 36.6 & 58.0 & 21.7 \\
w/o $\LM_{stat}$ & 71.4  & 63.6  & 36.6 & 33.1 & 56.9 & 15.4 \\
w/o $\LM_{pc}$ & 72.4 & 67.0 & 44.1 & 35.7 & 57.6 & 21.0\\
w/o asym. & 72.5 & 67.2 & 44.5 & 36.5 & 57.8 & 21.5\\
w/ asym. (opposite) & 71.4 & 66.9 & 44.0 & 34.8 & 56.7 & 20.8\\
\rowcolor{green!20}
FIRE & \bf 73.4 & \bf 68.6 & \bf 47.0 & \bf 36.9 & \bf 58.2 & \bf 22.0 \\ \bottomrule
\end{tabular}
\end{wraptable}

\noindent\textbf{Number of trainable parameters. }We evaluate the number of learnable parameters of different TTA methods, including normalization layers (Tent), attention (READ), adapters (TSA), prompts (BriMPR), and our proposed editors. As shown in Table \ref{tab:param}, since parameters in our approach mainly come from low-rank matrices, it remains parameter-efficient, even when editors are inserted in a layer-wise manner. Even with a rank as low as 2, FIRE still achieves the best performance on Kinetics50-C dataset.

\noindent\textbf{Ablation of FIRE. }We analyze the contributions of individual components in FIRE, including the FFT, each self-training loss, and the asymmetric design for the prediction consistency loss. For the prediction consistency loss, we replace the original asymmetric weighting with either average weighting, w/o asym., or reversed asymmetric weighting, w/ asym. (opposite). The results are reported in Table \ref{tab:ablation}. From Table \ref{tab:ablation}, the representation editors effectively mitigate distribution shifts, and FFT along the hidden dimension further improves performance. Among the self-training objectives, the source-target statistical alignment loss $\LM_{stat}$ yields the largest performance gain, supporting our motivation for layer-wise representation calibration. Moreover, both average weighting and the reversed asymmetric setting degrade performance, validating the effectiveness of the asymmetric strategy.

\section{Conclusion}
In this work, we propose FourIer Representation Editor (FIRE), a novel multimodal TTA approach that addresses distribution shifts by directly editing intermediate representations. Motivated by the observation that representation shifts emerge in shallow layers and increase with depth, FIRE performs layer-wise representation editing. Specifically, we insert an editor after each encoder transformer layer while keeping the pretrained model frozen. Subsequently, we introduce FFT to enrich the editing subspace and improve adaptation stability. Furthermore, we optimize these editors with multi-level self-training objectives to promote cross-modal semantic alignment and source-target statistical alignment. Extensive experiments demonstrate the superiority of FIRE.

\noindent\textbf{Limitations: }While FIRE performs representation editing across all intermediate layers, it assigns equal importance to each layer. Exploring adaptive layer selection may improve flexibility. We leave a more comprehensive investigation as future work.

\bibliography{iclr2027_conference}
\bibliographystyle{iclr2027_conference}

\clearpage
\appendix
\renewcommand{\thefigure}{\Roman{figure}}
\renewcommand{\thetable}{\Roman{table}}

\setcounter{figure}{0}
\setcounter{table}{0}
\setcounter{theorem}{0} 
\setcounter{lemma}{0}

\section{Analysis of Representation Editing}
We analyze the stability and effectiveness of the proposed method through the following two theorems.

\noindent{\textbf{Notation.}} The editing subspace $\hat{\U} \in \RB^{d_m\times d_n}$ satisfies $\hat{\U} \hat{\U}^{\top}=\I_{d_m}$. Let orthogonal projection $\P=\hat{\U}^{\top} \hat{\U}$ and define subspace $\SM=\operatorname{range}(\hat{\U}^{\top})$. Define $\h'=\W \h + \b$.

\begin{theorem}[Stability]
Assuming that the frozen downstream network $g$ is $L$-Lipschitz. For every $\h,\h_1,\h_2\in\RB^{d_n}$, then $\|\psi(\h_1)-\psi(\h_2)\|_2 \le (1+\|\W-\hat{\U}\|_2) \|\h_1-\h_2\|_2$ and $\|g(\psi(\h))-g(\h)\|_2 \le L(\|\W-\hat{\U}\|_2 \|\h\|_2+\|\b\|_2)$.
\end{theorem}

\begin{proof} 
Define $\delta=\h_1-\h_2$. Since the bias terms cancel out in the difference, we can obtain:
\begin{equation}
\begin{aligned}
\psi(\h_1)-\psi(\h_2)=\delta+\hat{\U}^{\top}(\W-\hat{\U})\delta.
\end{aligned}
\end{equation}
By $\|\hat{\U}^{\top}\|_2=1$ and the triangle inequality, we have:
\begin{equation}
\begin{aligned}
\|\psi(\h_1)-\psi(\h_2)\|_2 \le (1+\|\W-\hat{\U}\|_2) \|\h_1-\h_2\|_2.
\end{aligned}
\end{equation}

Since the downstream network $g$ is $L$-Lipschitz, we have:
\begin{equation}
\begin{aligned}
\|g(\psi(\h))-g(\h)\|_2 &\le L(\|\psi(\h)-\h \|_2) \\
&\le L(\|\h'-\hat{\U}\h \|_2) \\
&\le L(\|(\W-\hat{\U})\h+\b \|_2) \\
&\le L(\|\W-\hat{\U} \|_2 \|\h \|_2+\|\b \|_2).
\end{aligned}
\end{equation}
Near the no-op initialization, both $\W-\hat{\U}$ and $\|\b \|_2$ remain small, yielding a quantifiable upper bound on the perturbation to the frozen model.

This completes the proof.
\end{proof}

\begin{theorem}[Wasserstein interpretation and bound]
Assume that the source and target representations follow diagonal Gaussian distributions. Then, the representation statistical alignment loss and the 2-Wasserstein distance between the corresponding source and target distributions differ by at most a constant factor $\sqrt{2}$.
\end{theorem}

\begin{proof}
To provide a probabilistic interpretation of the representation statistics alignment objective, we adopt a moment-matched diagonal Gaussian approximation for the layer-wise source and target representations. Importantly, the Gaussian assumption is used only for theoretical interpretation and is not required by the FIRE optimization procedure.

Fix a layer $l$ and modality $r$. Approximate the source and target representations using the following diagonal Gaussian distributions:
\begin{equation}
\begin{aligned}
\NM(\mu^{\SM},\textstyle\sum_{\SM})), \NM(\mu^{\TM},\textstyle\sum_{\TM}),
\end{aligned}
\end{equation}
where $\textstyle\sum_{\SM}=\operatorname{diag}((\sigma^{\SM})^2)$, $\textstyle\sum_{\TM}=\operatorname{diag}((\sigma^{\TM})^2)$ and the standard deviation vectors are element-wise nonnegative. The 2-Wasserstein distance $\WM_{2}$ between the Gaussian distributions is given by:
\begin{equation}
\begin{aligned}
\WM_{2}^{2}&=\|\mu^{\TM}-\mu^{\SM} \|_{2}^{2}+\operatorname{tr}(\textstyle\sum_{\SM}+\textstyle\sum_{\TM}-2(\textstyle\sum_{\TM}^{\frac{1}{2}} \textstyle\sum_{\SM} \textstyle\sum_{\TM}^{\frac{1}{2}})^{\frac{1}{2}}) \\
&=\|\mu^{\TM}-\mu^{\SM} \|_{2}^{2} + \|\sigma^{\TM}-\sigma^{\SM} \|_{2}^{2}.
\end{aligned}
\end{equation}

Let $a=\|\mu^{\TM}-\mu^{\SM} \|_2$ and $b=\|\sigma^{\TM}-\sigma^{\SM} \|_2$. Then $\WM_{2}=\sqrt{a^2+b^2}$ and $D_{stat}=a+b$. Since $a\geq0$ and $b\geq0$, we have $\sqrt{a^2+b^2} \leq a+b$. By the Cauchy–Schwarz inequality,
\begin{equation}
\begin{aligned}
a+b \leq \sqrt{2}\sqrt{a^2+b^2},
\end{aligned}
\end{equation}
we can obtain:
\begin{equation}
\begin{aligned}
\WM_{2} \leq D_{stat} \leq \sqrt{2}\WM_{2}.
\end{aligned}
\end{equation}
Averaging over layers and modalities preserves the direction of the inequality. Therefore, the representation statistical alignment loss and the 2-Wasserstein distance between the corresponding source and target distributions differ by at most a constant factor $\sqrt{2}$. The editors can control the optimal transport distance between the source and target domains by minimizing the loss in Equation \ref{loss:stat}.

This completes the proof.
\end{proof}

\section{Algorithm of FIRE}
We summarize the overall optimization procedure of our method in Algorithm \ref{alg:FIRE}.

\begin{algorithm}[t]
\caption{Optimization algorithm of our approach}
\label{alg:FIRE}
\textbf{Input:} The pretrained model
$F_{\Theta}=\{f_{\Theta^a},f_{\Theta^v},f_{\Theta^m},f_{\Theta^c}\}$, representation editors $\{\psi_{a}^{l},\psi_{v}^{l}\}_{l=1}^{n}$, and the unlabeled target dataset $\DM_{\TM}=\{ (\x^{a}_j, \x^{v}_j) \}^{N_{\TM}}_{j=1}$.\\
\textbf{Output:} The adapted editors and the prediction.

\begin{algorithmic}[1]
\FOR{each mini-batch in $\DM_{\TM}$}
\STATE $\hat{\h}_{a}^{0} = \operatorname{Embed}(\x^a)$, $\hat{\h}_{v}^{0} = \operatorname{Embed}(\x^v)$;
\FOR{$l=1,\ldots,n$}
\STATE // Extract representations
\STATE $\h_{a}^{l}=f_{\Theta^a}^{l}(\hat{\h}_{a}^{l-1})$ and $\h_{v}^{l}=f_{\Theta^v}^{l}(\hat{\h}_{v}^{l-1})$; \\
\STATE // Edit representations
\STATE $\hat{\h}_{a}^{l}=\psi_{a}^{l}(\h_{a}^{l})$ and $\hat{\h}_{v}^{l}=\psi_{v}^{l}(\h_{v}^{l})$. \\
\STATE Compute statistics $(\mu^{\TM,a}_{l}, \sigma^{\TM,a}_{l}, \mu^{\TM,v}_{l}, \sigma^{\TM,v}_{l})$;
\ENDFOR
\STATE Obtain the prediction $\p=f_{\Theta^c}(f_{\Theta^m}(\hat{\h}_{a}^{n},\hat{\h}_{v}^{n}))$.\\
\STATE Construct modality-masked inputs and obtain predictions $\p^a$ and $\p^v$.\\
\STATE Calculate the final loss in Equation \ref{loss:final}.\\
\STATE Update the editors using back-propagation.
\ENDFOR
\end{algorithmic}
\end{algorithm}

\section{Notation Definition}
We summarize the notation definitions we used in this paper in Table \ref{tab:notation}.

\begin{table}[t]
\centering
\caption{Notation Definition}
\label{tab:notation}
\begin{tabular}{l|l}
\hline
\textbf{Notation} & \textbf{Description} \\
\hline
$N_{\SM}$              & The number of source data. \\
$N_{\TM}$              & The number of target data. \\
$y_i$                    & Category label of $i$-th data. \\
$\x_i^a / \x_i^v$      & Audio / Video data point. \\
$\DM_{\SM} / \DM_{\TM}$          & Source / Target domain. \\
$r$        & $r$-modality. \\
$F_{\Theta}$               & Pretrained model. \\
$f_{\Theta^a}/ f_{\Theta^v}$ & Audio / Video encoder. \\
$f_{\Theta^m}$ & Fusion layer. \\
$f_{\Theta^c}$ & Classifier head. \\
$\Theta$          & Model parameters. \\
$P_{\SM} / P_{\TM}$   & Source / Target distribution. \\
$\Psi(\cdot)$   & Representation editor. \\
$\psi(\cdot)$   & Fourier representation editor. \\
$\U / \W / \b$   & Editor parameters. \\
$\FM(\cdot)$   & Discrete Fourier Transform. \\
$\mathfrak{R}$   & Real component. \\
$\hat{\U}$   & Fourier-enhanced subspace. \\
$\operatorname{qf}(\cdot)$ & Orthogonal factor $Q$ in QR decomposition. \\
$B$ & Batch size. \\
$sim(\cdot,\cdot)$ & Cosine similarity. \\
$\LM_{cl}$ & Cross-modal semantic alignment loss. \\
$\tau$ & Temperature hyperparameter in $\LM_{cl}$. \\
$\mu / \sigma$   & Representation mean / standard deviation. \\
$n$   & The number of encoder layers. \\
$\LM_{stat}$    & Source-target statistical alignment loss. \\
$\p$    & Prediction with complete modalities. \\
$\p^{r}$    & Prediction after masking modality $r$. \\
$\ell(\cdot)$    & Cross-entropy loss. \\
$\LM_{pc}$    & Asymmetric prediction consistency loss. \\
\bottomrule
\end{tabular}
\end{table}

\section{Experiments Details}
\subsection{Dataset Details}
\noindent{\textbf{Kinetics50:}} The Kinetics \cite{KS:journals/corr/KayCSZHVVGBNSZ17} dataset is a large-scale and high-quality dataset for human action recognition in videos. The dataset consists of around 300,000 video clips covering 400 human action classes with at least 400 video clips for each action class. Each video clip lasts around 10 seconds and is labeled with a single action class. Following \cite{READ:conf/iclr/Yang0Z0024}, we use a subset of Kinetics which consists of 50 classes, 29,204 training pairs, and 2,466 test pairs.

\noindent{\textbf{VGGSound:}} VGGSound \cite{vgg:conf/icassp/ChenXVZ20} is a large-scale audio-visual correspondence dataset consisting of short clips of audio sounds, extracted from YouTube videos. All videos are captured in the wild and exhibit audio-visual correspondence, with the sound sources being visually apparent. Each video in this dataset has a fixed duration of 10 seconds.

\noindent{\textbf{Kinetics50-C and VGGSound-C:}} Following previous work \cite{READ:conf/iclr/Yang0Z0024}, we introduce 15 types of video corruptions and 6 types of audio corruptions. Video corruptions include “Gaussian Noise” (Gauss.), “Shot Noise” (Shot), “Impulse Noise” (Impul.), “Defocus Blur” (Defoc.), “Glass Blur” (Glass), “Motion Blur” (Motion), “Zoom Blur” (Zoom), “Snow” (Snow), “Frost” (Frost), “Fog” (Fog), “Brightness” (Bright.), “Contrast” (Contr.), “Elastic” (Elastic), “Pixelate” (Pixel) and “JPEG” (JPEG). Audio corruptions include “Gaussian Noise” (Gauss.), “Paris Traffic Noise” (Traff.), “Crowd Noise” (Crowd), “Rainy Noise” (Rain), “Thunder Noise” (Thund.) and “Windy Noise” (Wind).

\subsection{Baseline Details}
\noindent{\textbf{Tent:}} Tent \cite{Tent:conf/iclr/WangSLOD21} optimizes the affine parameters in the normalization layer by minimizing the entropy of the model’s predictions. In the transformer-based CAV-MAE, we replace the Batch Normalization (BN) layers in the original implementation with Layer Normalization (LN) layers.

\noindent{\textbf{T3A:}} T3A \cite{T3A:conf/nips/IwasawaM21} replaces the output layer of the predictor trained on the source domain with a pseudo-prototypical classifier and adjusts the prototype features during test time.

\noindent{\textbf{EATA:}} EATA \cite{EATA:conf/icml/NiuW0CZZT22} selects test samples with low entropy to participate in entropy minimization, and introduces a weighted Fisher regularizer. The entropy threshold $E_0$ is set to $0.4\times\ln C$ (where $C$ is the number of classes). The cosine similarity threshold $\epsilon$ used for filtering redundant samples is set to 0.1. The trade-off hyperparameter $\beta$ is set to 1. The Fisher information is calculated using 2,000 unlabeled in-distribution samples, and the moving average factor $\alpha$ is set to 0.1.

\noindent{\textbf{SAR:}} SAR \cite{SAR:conf/iclr/Niu00WCZT23} introduces sharpness-aware learning and minimizes entropy. The entropy threshold $E_0$ is set to $0.4\times\ln C$. The radius $\rho$ in the sharpness-aware optimization is set to 0.05. The model recovery threshold $e_0$ is set to 0.2. The moving average factor used to track the loss value is set to 0.9.

\noindent{\textbf{DeYO:}} DeYO \cite{DeYO:conf/iclr/0004JLPSHY24} introduces the Pseudo-Label Probability Difference (PLPD) metric to identify harmful samples that cannot be detected by entropy. The entropy threshold $\tau_{Ent}$ is set to $0.5\times\ln C$, the PLPD threshold $\tau_{PLPD}$ is set to $0.2\times\ln C$, and the normalization factor $Ent_0$ in the weighting function is set to $0.4\times\ln C$. The default patch-shuffling is used as the image transformation. 

\noindent{\textbf{FOA:}} FOA \cite{FOA:conf/icml/NiuMCWZ24} inserts prompts at the input level and employs the derivative-free covariance matrix adaptation (CMA) evolution strategy. The number of prompt embeddings $N_p$ is set to 1 with the default uniform initialization. The population size $K$ in the CMA evolution strategy is set to $27 = 4 + 3 \ln(2 \times 768)$. The trade-off parameter $\lambda$ in the fitness function is set to $0.4$. The step size $\gamma$ in the back-to-source activation shifting is set to 1.0. The moving average factor for computing the test statistics is set to 0.1.

\noindent{\textbf{READ:}} READ \cite{READ:conf/iclr/Yang0Z0024} updates the self-attention layer of the fusion module by optimizing the confidence-aware loss. The confidence threshold $\gamma$ is set to $e^{-1}$.

\noindent{\textbf{ABPEM:}} ABPEM \cite{ABPEM:conf/aaai/Zhao0LHY0025} aligns cross-attention to self-attention to reduce inter-modal differences, while excluding non-dominant class samples to reduce gradient noise in entropy loss.
The threshold $k$ for class ranking is set to 8/30 for Kinetics50-C and VGGSound-C. The weight $\lambda$ for the attention bootstrapping loss is set to 1.

\noindent{\textbf{SuMi:}} SuMi \cite{SuMi:conf/iclr/Guo025} selects high-quality samples via unimodal-assisted identification, and balances adaptation across modalities by leveraging mutual information sharing. The multimodal threshold $\gamma_m$ and the normalization factor $\text{Ent}_0$ are set to $0.4\times \ln C$. The unimodal threshold $\gamma_u$ is set to $e^{-1}$. The smoothing coefficient $\beta$ is set to 0.6/0.9, the weighting term $\lambda$ is set to 5.0, and the unimodal assistance $t$ is set to 1.0 by default for Kinetics50-C and VGGSound-C. For the multimodal shift setting and real-world shift setting, we set the mutual information sharing term $t_{0}$ as $iter/2$.

\noindent{\textbf{TSA:}} TSA \cite{TSA:conf/icml/ChenZHJWF0B25} introduces modality-specific adapters and a router module to select the modality to adapt. The coefficient of the self-training loss is set to 0.5, and the softmax temperature is set to 0.001.

\noindent{\textbf{PTA:}} PTA \cite{PTA:wang2026partition} consists of two key components: Partition and Debiased Reweighting (PDR) and multimodal Attention-Guided Alignment (AGA). The weight $s$ assigned to samples with large prediction bias is set to 0.5.

\noindent{\textbf{BriMPR:}} BriMPR \cite{BriMPR:conf/aaai/LiF26} employs prompt tuning to align the global feature distribution of each modality with its corresponding source distribution. The default number of prompts per layer is set to 10. $\tau_0$ and $D_0$ are set to 0.2 and 5, respectively. $\tau$ is set to 0.07/0.25 for the unimodal and multimodal corruption settings, respectively. 

\end{document}